\relax
\documentclass[letterpaper]{article} % DO NOT CHANGE THIS
\usepackage{aaai21}  % DO NOT CHANGE THIS
\usepackage{times}  % DO NOT CHANGE THIS
\usepackage{helvet} % DO NOT CHANGE THIS
\usepackage{courier}  % DO NOT CHANGE THIS
\usepackage[hyphens]{url}  % DO NOT CHANGE THIS
\usepackage{graphicx} % DO NOT CHANGE THIS
\usepackage{graphicx}  % DO NOT CHANGE THIS
\usepackage{natbib}  % DO NOT CHANGE THIS AND DO NOT ADD ANY OPTIONS TO IT
\usepackage{caption} % DO NOT CHANGE THIS AND DO NOT ADD ANY OPTIONS TO IT
\usepackage{amsmath,amssymb}
\usepackage{amsthm}
\usepackage{amsfonts}
\usepackage{subfigure}
\usepackage{enumitem}
\usepackage{algorithm}
\usepackage{algorithmic}
\usepackage[switch]{lineno}

\newtheorem{theorem}{Theorem}
\newtheorem{proposition}[theorem]{Proposition}

\title{Compressing Deep Convolutional Neural Networks by Stacking Low-dimensional Binary Convolution Filters}
\author {
        Weichao Lan,\textsuperscript{\rm 1}
        Liang Lan, \textsuperscript{\rm 1}
        \\
}
\affiliations {
    \textsuperscript{\rm 1}Department of Computer Science, Hong Kong Baptist University, Hong Kong SAR, China \\
    \{cswclan, lanliang\}@comp.hkbu.edu.hk

}

\begin{document}
% \linenumbers  
\maketitle

\begin{abstract}
Deep Convolutional Neural Networks (CNN) have been successfully applied to many real-life problems. However, the huge memory cost of deep CNN models poses a great challenge of deploying them on memory-constrained devices (e.g., mobile phones). One popular way to reduce the memory cost of deep CNN model is to train binary CNN where the weights in convolution filters are either $1$ or $-1$ and therefore each weight can be efficiently stored using a single bit. However, the compression ratio of existing binary CNN models is upper bounded by $\sim32$. To address this limitation, we propose a novel method to compress deep CNN model by stacking low-dimensional binary convolution filters. Our proposed method approximates a standard convolution filter by selecting and stacking filters from a set of low-dimensional binary convolution filters. This set of low-dimensional binary convolution filters is shared across all filters for a given convolution layer. Therefore, our method will achieve much larger compression ratio than binary CNN models. In order to train our proposed model, we have theoretically shown that our proposed model is equivalent to select and stack intermediate feature maps generated by low-dimensional binary filters. Therefore,  our proposed model can be efficiently trained using the split-transform-merge strategy. We also provide detailed analysis of the memory and computation cost of our model in model inference. We compared the proposed method with other five popular model compression techniques on two benchmark datasets. Our experimental results have demonstrated that our proposed method achieves much higher compression ratio than existing methods while maintains comparable accuracy.
\end{abstract}

\section{Introduction} \label{intro}
Recent advances in deep convolutional neural network (CNN) have produced powerful models that achieve high accuracy on a wide variety of real-life tasks. %in computer vision, such as image classification \cite{krizhevsky2012imagenet}, image segmentation \cite{long2015fully} and image detection \cite{ren2015faster}. 
These deep CNN models typically consist of a large number of convolution layers involving many parameters. They require large memory to store the model parameters and intensive computation for model inference. Due to concerns on privacy, security and latency caused by performing deep CNN model inference remotely in the cloud, deploying deep CNN models on edge devices (e.g., mobile phones) and performing local on-device model inference has gained growing interests recently \cite{zhang2018shufflenet, howard2017mobilenets}. However, the huge memory cost of deep CNN model poses a great challenge when deploying it on resource-constrained edge devices. For example, the VGG-16 network \cite{simonyan2014very}, which is one of the famous deep CNN models, performs very well in both image classification and object detection tasks. But this VGG-16 network requires more than 500MB memory and over 15 billions floating number operations (FLOPs) to classify a single input image \cite{cheng2018recent}. 

To reduce the memory and computation cost of deep CNN models, several model compression methods have been proposed in recent years. These methods can be generally categorized into five major types: (1) parameter pruning and sharing \cite{han2015learning, han2015deep, ullrich2017soft}: pruning redundant, non-informative weights in pre-trained CNN models; (2)
low-rank approximation \cite{denton2014exploiting, jaderberg2014speeding}: finding appropriate low-rank approximation for convolution layers; (3) knowledge distillation \cite{ba2014deep,hinton2015distilling,bucilua2006model}: approximating deep neural networks with shallow models; (4) compact convolution filters \cite{howard2017mobilenets, zhang2018shufflenet}: using carefully designed structural convolution filters; and (5) model quantization \cite{han2015deep, gupta2015deep}: quantizating the model parameters and therefore reducing the number of bits to represent each weight. Among these existing studies, model quantization is one of the most popular ways for deep CNN model compression. It is widely used in commercial model deployments and has several advantages compared with other methods \cite{krishnamoorthi2018quantizing}: (1) broadly applicable across different network architectures and hardwares; (2) smaller model footprint; (3) faster computation and (4) powerful efficiency.  %   is probably the most important way

%In model quantization, 
Binary neural networks is the extreme case in model quantization where each weight can only be $1$ or $-1$ and therefore can be stored using a single bit. In the research direction of binary neural networks, the pioneering work BinaryConnect (BC) proposed by \cite{courbariaux2015binaryconnect} is the first successful method that incorporates learning binary model weights in the training process. Several extensions to BC have been proposed, such as Binarized Neural Networks(BNN) presented by \cite{hubara2016binarized}, Binary Weight Network (BWN) and XNOR-Networks (XNOR-Net) proposed by \cite{rastegari2016xnor}. Even though the existing works on binary neural networks have shown promising results on model compression and acceleration, they use a binary filter with the same kernel size and the same filter depth as a standard convolution filter. Therefore, for a given popular CNN architecture, binary neural networks can only compress the original model by up to $\sim 32$ times. This upper bound on compression ratio (i.e., 32) could limit the applications of binary CNNs on resource-constrained devices, especially for large scale CNNs with a huge number of parameters.

%To overcome the theoretical compression ratio limit of binary CNN models, 
Motivated by recent work LegoNet \cite{yang2019legonet} which constructs efficient convolutional networks with a set of small full-precision convolution filters named lego filters, we propose to compress deep CNN by selecting and stacking low-dimensional binary convolution filters. In our proposed method, each original convolution filter is approximated by stacking a number of filters selected from a set of low-dimensional binary convolution filters. This set of low-dimensional binary convolution filters is shared across all convolution filters for a given convolution layer. Therefore, our proposed method can achieve much higher compression ratio than binary CNNs. Compared with LegoNet, our proposed method can reduce the memory cost of LegoNet by a factor of $\sim 32$ since our basic building blocks are binary filters instead of full-precision filters. 

The main contributions of this paper can be summarized as follows: First, we propose a novel method to overcome the theoretical compression ratio limit of recent works on binary CNN models. Second, we have shown that our proposed model can be reformulated as selecting and stacking feature maps generated by low-dimensional binary convolution filters. After reformulation, our proposed model can be efficiently trained using the split-transform-merge strategy and can be easily implemented by using any existing deep learning framework (e.g., PyTorch or Tensorflow). Third, we provide detailed analysis of the memory and computation cost of our model for model inference. Finally, we compare our proposed method with other five popular model compression algorithms on three benchmark datasets. Our experimental results clearly demonstrate that our  proposed  method can achieve comparable accuracy with much higher compression ratio. We also empirically explore the impact of various training techniques (e.g., choice of optimizer, batch normalization) on our proposed method in the experiments.

%\section{Preliminaries on Convolutional Neural Networks and Binary Convolutional Neural networks}\label{sec:prelim}
\section{Preliminaries}\label{sec:prelim}
\textbf{Convolutional Neural Networks.} In a standard CNN, convolution operation is the basic operation. As shown in Fig. \ref{Fig.1.1}, for a given convolution layer in CNN, it transforms a three-dimensional input tensor $\mathbf{X}_{input} \in \mathbb{R}^{w_{in} \times h_{in} \times c_{in}}$, where $w_{in}$, $h_{in}$ and $c_{in}$ represents the width, height and depth (or called number of channels) of the input tensor, into a three-dimensional output tensor $\mathbf{X}_{output} \in \mathbb{R}^{w_{out} \times h_{out} \times c_{out}}$ by

\begin{equation}\label{eq:conv}
\mathbf{X}_{output} = \text{Conv}(\mathbf{X}_{input}, \mathbf{W}),
\end{equation}
where Conv() denotes the convolution operation. Each entry in the output tensor $\mathbf{X}_{output}$ is obtained by an element-wise multiplication between a convolution filter $\mathbf{W}^{i} \in \mathbb{R}^{d \times d \times c_{in}} $ and a patch $\mathbf{X}_{input}^{i} \in \mathbb{R}^{d \times d \times c_{in}}$ extracted from $\mathbf{X}_{input}$ followed by summation. $d \times d$ is the kernel size of the convolution filter (usually $d$ is 3) and $c_{in}$ is depth of the convolution filter which is equal to the number of input channels. Therefore, for a given convolution layer with $c_{out}$ convolution filters, we can use $\mathbf{W} \in \mathbb{R}^{d \times d \times c_{in} \times c_{out}}$ to denote the parameters needed for all $c_{out}$ convolution filters. The memory cost of storing weights of convolution filters $\mathbf{W}$ for a given layer is $ d \times d \times c_{in} \times c_{out}\times 32$ bits assuming 32-bit floating-point values are used to represent model weights. It is high since deep CNN models usually contain a large number of layers. The computation cost for CNN model inference is also high because the convolution operation involves a large number of FLOPs.

\noindent\textbf{Binary Convolutional Neural networks}. To reduce the memory and computation cost of deep CNN model, several algorithms \cite{simonyan2014very, courbariaux2015binaryconnect, rastegari2016xnor, hubara2016binarized, AlizadehFLG19an} have been proposed recently. Their core idea is to binarize the model weights $\mathbf{W} \in \mathbb{R}^{d \times d \times c_{in} \times c_{out}}$. Since a binary weight can be efficiently stored with a single bit, these methods can reduce the memory cost of storing $\mathbf{W} \in \mathbb{R}^{d \times d \times c_{in} \times c_{out}}$ to $d \times d \times c_{in} \times c_{out}$ bits. It has been shown that these methods can achieve good classification accuracy with much less memory and computation cost compared to standard CNN model. However, due to that the binarized $\mathbf{W}$ is still of size $d \times d \times c_{in} \times c_{out}$, these binary CNNs can only reduce the memory cost of deep CNN model by up to $\sim 32$ times.

% \begin{figure*}[htb]
% \centering
% \subfigure[convolution Filters $\mathbf{W}$]{\centering \label{Fig.1.1}\includegraphics[width=0.9\columnwidth]{convFilters.png}}
% \subfigure[convolution Filters $\mathbf{W}$ Approximated by Stacking Low-dimensional Binary convolution Filters]{\centering \label{Fig.1.2}\includegraphics[width=0.85\columnwidth]{convFiltersBySmallBF.png}}
% \caption{Approximating convolution Filters by Stacking Low-dimensional Binary convolution Filters}
% \label{Fig1}
% \end{figure*}

\begin{figure}[htb]
\centering
\subfigure[convolution filters $\mathbf{W}$]{\centering \label{Fig.1.1}\includegraphics[width=1.0\columnwidth]{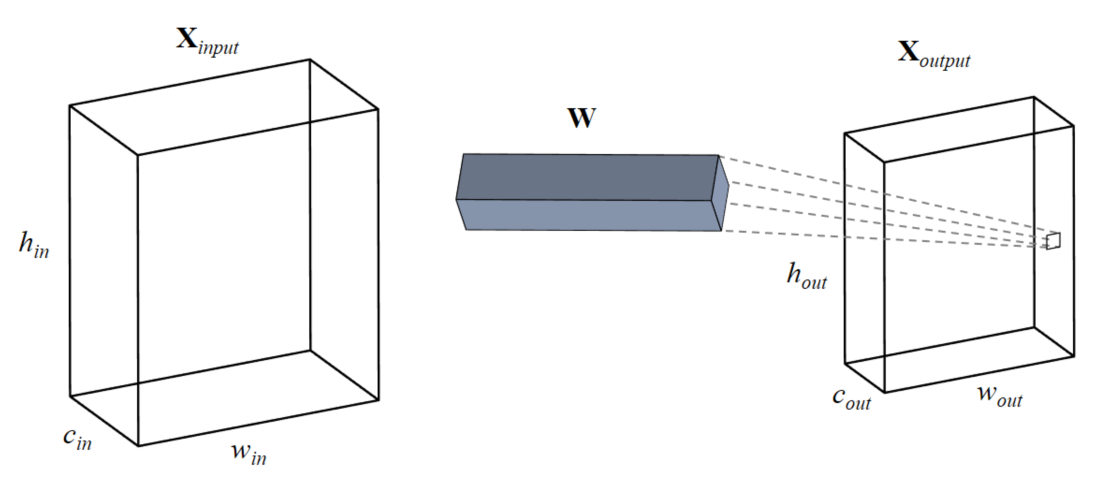}}
\subfigure[stacked $\mathbf{W}$ using low-dimensional binary convolution filters]{\centering \label{Fig.1.2}\includegraphics[width=1.0\columnwidth]{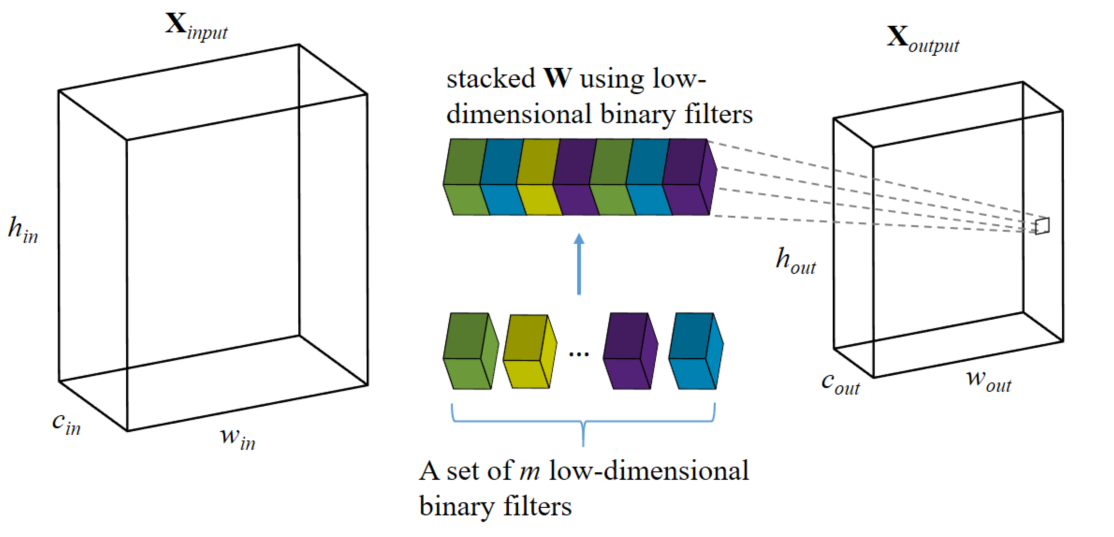}}
\caption{Approximating Convolution Filters by Stacking Low-dimensional Binary convolution Filters}
\label{Fig1}
\end{figure}

%\section{Compressing Deep Convolutional Neural Network by Stacking low-dimensional Binary convolution Filters}
\section{Methodology}

In this section, we propose a new method that can overcome the theoretical compression ratio limit of binary CNN models. Instead of approximating convolution filters using binary convolution filters with the same kernel size and the same filter depth, our proposed idea approximates the convolution filters by selecting and stacking a number of filters from a set of low-dimensional binary convolution filters. The depth of these binary filters will be much smaller than the depth of original convolution filters. Therefore, we call them \textit{low-dimensional binary convolution filters} in this paper. This set of low-dimensional binary convolution filters is shared across all convolution filters for a given convolution layer. The main idea of our proposed method is illustrated in Fig. \ref{Fig1} and we will explain the details of it in following subsections.

\subsection{Approximating convolution Filters by Stacking low-dimensional Binary Filters}\label{sec:modelformulation}
Suppose we use $\mathbf{W}^{t} \in \mathbb{R}^{d \times d \times c_{in}}$ to denote the $t$-th full-precision convolution filter in a convolution layer in a standard CNN. According to (\ref{eq:conv}), the $t$-th feature map in the output tensor $\mathbf{X}^t_{output} \in \mathbb{R}^{w_{out} \times h_{out}}$ generated by convolution filter $\mathbf{W}^{t}$ can be written as
\begin{equation}\label{eq:conv_l}
\mathbf{X}_{output}^{t} = \text{Conv}(\mathbf{X}_{input}, \mathbf{W}^t).
\end{equation}

Let us use $\{\mathbf{B}_1, \mathbf{B}_2, \dots, \mathbf{B}_m\}$ to denote a set of $m$ shared binary convolution filters for a given convolution layer. $\mathbf{B}_i \in \mathbb{R}^{d \times d \times s}$ denotes the weights for the $i$-th binary filter where $s$ is depth of the binary convolution filters. In here, the depth $s$ is much smaller than $c_{in}$ which is the depth of original convolution filters. Each element in $\mathbf{B}_i$ is either $1$ or $-1$. We propose to approximate $\mathbf{W}^{t}$ by selecting $k = \frac{c_{in}}{s}$ low-dimensional binary convolution filters from $\{\mathbf{B}_1,\dots, \mathbf{B}_m\}$ and then stacking them together. Let us define an indicator matrix $\mathbf{P} \in \mathbb{R}^{m \times k}$ as
\begin{equation}\label{eq:matrixP}
\mathbf{P}_{ji} =
\begin{cases}
   \ 1 \ \ \ \ \text{if the $i$-th block of $\mathbf{W}^{t}$ is $\mathbf{B}_j$  } \\
   \ 0 \ \ \ \ \text{otherwise}\\
  \end{cases}.
\end{equation}
By following the selecting and stacking idea, $\mathbf{W}^{t}$ will be approximated by [$\sum_{j=1}^{m}\mathbf{P}_{j1}\mathbf{B}{j}, \sum_{j=1}^{m}\mathbf{P}_{j2}\mathbf{B}{j},\dots, \sum_{j=1}^{m}\mathbf{P}_{jk}\mathbf{B}{j}$] which concatenates $k$ low dimensional binary convolution filters together in column-wise manner. Here we can also introduce another variable $\alpha_i$ to denote the scaling factor associated to the $i$-th block of $\mathbf{W}^{t}$ when we concatenate different binary convolution filters together, that is, $\mathbf{W}^{t} \approx$ [$\alpha_1\sum_{j=1}^{m}\mathbf{P}_{j1}\mathbf{B}{j}, \alpha_2\sum_{j=1}^{m}\mathbf{P}_{j2}\mathbf{B}{j}, \dots,$ $\alpha_k\sum_{j=1}^{m}\mathbf{P}_{jk}\mathbf{B}{j}$]. We can treat $\alpha_i\mathbf{P}_{ji}$ as single variable by changing the definition of $\mathbf{P}$ in (\ref{eq:matrixP}) to
\begin{equation}\label{eq:matrixP_2}
\mathbf{P}_{ji} =
\begin{cases}
   \ \alpha_i \ \ \ \ \text{if the $i$-th block of $\mathbf{W}^{t}$ is $\mathbf{B}_j$  } \\
   \ 0 \ \ \ \ \ \text{otherwise}\\
  \end{cases}.
\end{equation}
In our experiment section, we have shown that introducing the scaling factors $\{\alpha_i\}_{i=1}^{k}$ always obtains slightly better classiﬁcation accuracy than without using them.

Let us split the $\mathbf{X}_{input} \in \mathbb{R}^{w_{in} \times h_{in} \times c_{in}}$ into $k = \frac{c_{in}}{s}$ parts \{$\mathbf{X}_{input(1)}$, $\mathbf{X}_{input(2)}$, \dots, $\mathbf{X}_{input(k)}$ \} where the size of each part $\mathbf{X}_{input(i)}$ is $\mathbb{R}^{w_{in} \times h_{in} \times s}$. Then the $t$-th feature map in the output tensor generated by convolution filter $\mathbf{W}^{t}$ as shown in (\ref{eq:conv_l}) can be approximated as
\begin{equation}\label{eq:conv_l_approximation}
\mathbf{X}_{output}^{t} = \sum_{i = 1}^{k}\text{Conv}(\mathbf{X}_{input(i)}, \sum_
{j=1}^{m}\mathbf{P}_{ji}\mathbf{B}{j}).
\end{equation}

Note that $\|\mathbf{P}_{(:,i)}\|_0 = 1$ (i.e., each column of $\mathbf{P}$ only contains one non-zero value) means that only one binary filter $\mathbf{B}_j$ is selected to perform the convolution operation on the $i$-th part of $\mathbf{X}_{input}$. The $\mathbf{X}_{output}^{t}$ is a element-wise sum of $k$ feature maps. Each feature map is generated by applying a single low-dimensional binary convolution filter to one part of the input.

As shown in (\ref{eq:conv_l_approximation}), for a convolution filter in a given convolution layer, the model parameters are $\{\mathbf{B}_1,\dots, \mathbf{B}_m\}$ and $\mathbf{P}$, where $\{\mathbf{B}_1,\dots, \mathbf{B}_m\}$ is shared by all convolution filters for a given convolution layer. Therefore, the model parameters of our proposed method for a given convolution layer with $c_{out}$ convolution filters are just $\{\mathbf{B}_1,\dots, \mathbf{B}_m\}$ and $\{\mathbf{P}^{t}\}_{t=1}^{c_{out}}$. By considering that the memory cost of storing $\{\mathbf{P}^{t}\}_{t=1}^{c_{out}}$ is relatively small than storing $\{\mathbf{B}_1,\dots, \mathbf{B}_m\}$, our proposed method can significantly reduce the memory cost of binary CNNs. A detailed analysis of the compression ratio and computation cost of our proposed method will be provided in section \ref{sec:algorithmAnalysis}.
% our proposed method can achieve compression ratio as
% \begin{equation}\label{eq:compressionRatio}
% \sim \frac{d\times d \times c_{in} \times c_{out} \times 32}{d \times d \times s \times m} = \frac{c_{in} \times c_{out} \times 32}{s \times m}.
% \end{equation}
% In our hyperparameter setting, we set $s = c_{in}f_1$ and $m = c_{out}f_2$ where $f_1$ and $f_2$ are fractional numbers less than 1 (e.g., $\frac{1}{2}$, $\frac{1}{4}$ or $\frac{1}{8}$ etc). Therefore, our proposed method can overcome the theoretical  compression  ratio  limit  of binary CNNs by a factor of $\frac{1}{f_1f_2}$. 

\subsection{Training Model Parameters of the Proposed Compressed CNN}\label{sec:modelTraining}
In this section, we present our algorithm to learn the model parameters $\{\mathbf{B}_1,\dots, \mathbf{B}_m\}$ and $\{\mathbf{P}^t\}_{t=1}^{c_{out}}$ from the training data. Without loss of generality, let us consider to optimize the model parameters for one layer. Assume $\{\mathbf{X}_{input}, \mathbf{Y}_{output}\}$ is a mini-batch of inputs and targets for a given convolution layer. Therefore, the objective for optimizing $\{\mathbf{B}_1,\dots, \mathbf{B}_m\}$ and $\{\mathbf{P}^t\}_{t=1}^{c_{out}}$ will be
\begin{equation}\label{eq:objective}
\begin{split}
%\min \limits_{\{\mathbf{B}_1,\dots, \mathbf{B}_m\},\{\mathbf{P}^t\}_{t=1}^{c_{out}}} & \sum_{t=1}^{c_{out}}\|\mathbf{Y}_{output}^{t} -  \sum_{i=1}^{k}\text{Conv}(\mathbf{X}_{input(i)}, \sum_{j=1}^{m}\mathbf{P}_{ji}^{t}\mathbf{B}_j)\|^2  \\
\min  & \sum_{t=1}^{c_{out}}\|\mathbf{Y}_{output}^{t} -  \sum_{i=1}^{k}\text{Conv}(\mathbf{X}_{input(i)}, \sum_{j=1}^{m}\mathbf{P}_{ji}^{t}\mathbf{B}_j)\|^2  \\
	s.t & \ \|\mathbf{P}_{(:,i)}^{t}\|_0 = 1 \\
    %    & \ \sum_{j=1}^{m}(\mathbf{P})_{ji}^{l} = 1, \\
        & \ \mathbf{B}_{ij} \in \{-1, 1\}.
\end{split}
\end{equation}

In order to optimize (\ref{eq:objective}), we first prove that the convolution operation $\text{Conv}(\mathbf{X}_{input(i)}, \sum_{j=1}^{m}\mathbf{P}_{ji}^{t}\mathbf{B}_j)$ is equivalent to $\sum_{j=1}^{m}\mathbf{P}_{ji}^{t}\text{Conv}(\mathbf{X}_{input(i)}, \mathbf{B}_j)$ as shown in Proposition \ref{proposition:equivalance}. In other words, selecting a convolution filter (i.e., $\sum_{j=1}^{m}\mathbf{P}_{ji}^{t}\mathbf{B}_j$) and then performing convolution operation is equivalent to performing $m$ convolution operations and then selecting a feature map from the generated $m$ intermediate feature maps $\{\text{Conv}(\mathbf{X}_{input(i)}, \mathbf{B}_j)\}_{j=1}^{m}$. The advantage of latter computation is that it can reduce the computation cost since these $m$ intermediate feature maps $\{\text{Conv}(\mathbf{X}_{input(i)}, \mathbf{B}_j)\}_{j=1}^{m}$ is shared across all $c_{out}$ convolution filters for a given convolution layer.

\begin{proposition}\label{proposition:equivalance}
Suppose $\mathbf{X}_{input(i)} \in \mathbb{R}^{w_{in} \times h_{in} \times s}$, $\{\mathbf{B}_1,\dots, \mathbf{B}_m\}$ is a set of $m$ low-dimensional binary filters where each $\mathbf{B}_i \in \mathbb{R}^{d \times d \times s}$ and $\mathbf{P}_{(:,i)}^{t}$ is the $i$-th column in $\mathbf{P}^t$ which is a length-$m$ sparse vector with only one non-zero element. Then, $\emph{Conv}(\mathbf{X}_{input(i)}, \sum_{j=1}^{m}\mathbf{P}_{ji}^{t}\mathbf{B}_j)$ is equivalent to $\sum_{j=1}^{m}\mathbf{P}_{ji}^{t}\emph{Conv}(\mathbf{X}_{input(i)}, \mathbf{B}_j)$.
\end{proposition}

% \begin{proof}
% Let us divide $\mathbf{X}_{input(i)}$ into $p = w_{out} \times h_{out}$ patches and each patch $\mathbf{G}_j$ is with size $\mathbb{R}^{d \times d \times s}$. We can vectorize each patch and form a matrix $\widetilde{\mathbf{X}} = [\text{vec}(\mathbf{G}_1), \text{vec}(\mathbf{G}_2), \dots, \text{vec}(\mathbf{G}_p)] \in \mathbb{R}^{d^2s \times p}$. Similarly, we can vectorize each low-dimensional binary convolution filter $\mathbf{B}_i$ and form a matrix $\widetilde{\mathbf{B}} = [\text{vec}(\mathbf{B}_1), \text{vec}(\mathbf{B}_2), $ $\dots, \text{vec}(\mathbf{B}_m)] \in \mathbb{R}^{d^2s \times m}$. Based on the definition of convolution operation and applying the associative property of matrix multiplication, we have
% \begin{equation}\label{eq:reform_conv}
% \text{Conv}(\mathbf{X}_{input(i)}, \sum_{j=1}^{m}\mathbf{P}_{ji}^{t}\mathbf{B}_j) =  \widetilde{\mathbf{X}}^{T}(\widetilde{\mathbf{B}}\mathbf{P}_{(:,i)}^{t}) = (\widetilde{\mathbf{X}}^{T}\widetilde{\mathbf{B}})\mathbf{P}_{(:,i)}^{t}.
% \end{equation}
% Note that $(\widetilde{\mathbf{X}}^{T}\widetilde{\mathbf{B}})\mathbf{P}_{(:,i)}^{t}$ in (\ref{eq:reform_conv}) can be rewritten as $\sum_{j=1}^{m}\mathbf{P}_{ji}^{t}\text{Conv}(\mathbf{X}_{input(i)}, \mathbf{B}_j)$ which can be interpreted as we first perform convolution operations on $\mathbf{X}_{input(i)}$ using  $\{\mathbf{B}_1,\dots, \mathbf{B}_m\}$ to generate $m$ intermediate feature maps $\{\text{Conv}(\mathbf{X}_{input(i)},$ $\mathbf{B}_j)\}_{j=1}^{m}$ and then select one feature map from them using sparse vector $\mathbf{P}^t_{(:,i)}$.
% \end{proof}

\begin{figure}[htbp]
\centering
\includegraphics[width=0.9\columnwidth]{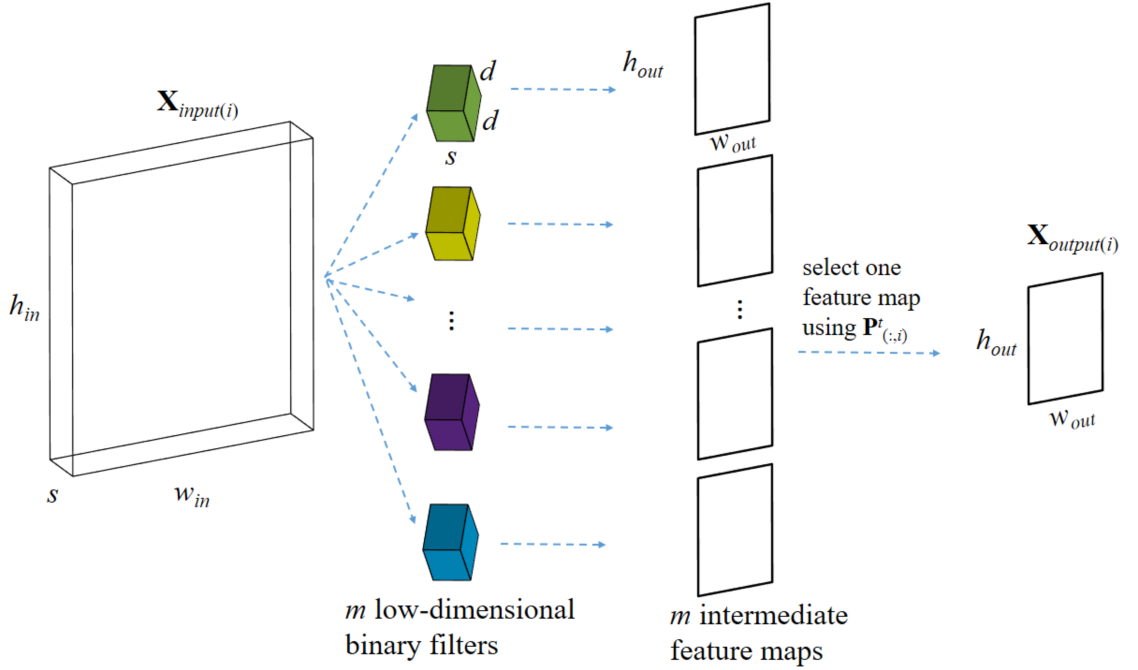}
\caption{Reformat as convolution and then select feature}
\label{fig2}
\end{figure}

The proof of Proposition \ref{proposition:equivalance} can be done by using the definition of convolution operation and the associative property of matrix multiplication. The details can be found in supplementary materials. Based on Proposition \ref{proposition:equivalance}, \(\text{Conv}(\mathbf{X}_{input(i)}, \sum_{j=1}^{m}\mathbf{P}_{ji}^{t}\mathbf{B}_j)\) in (\ref{eq:objective}) can be reformulated as: (1) first performing convolution operations on $\mathbf{X}_{input(i)}$ using $\{\mathbf{B}_1,\dots, \mathbf{B}_m\}$ to generate $m$ intermediate feature maps; (2) selecting one feature map from them. This procedure is also illustrated in Figure \ref{fig2}. After reformulation, our proposed model can be efficiently trained using the split-transform-merge strategy as in \citet{szegedy2015going}.

Similar to training a standard CNN, the training process of our proposed model involves three steps in each iteration: (1) forward propagation; (2) backward propagation and (3) parameter update. In our proposed model, we have additional non-smooth constraints on $\{\mathbf{B}_1,\dots, \mathbf{B}_m\}$ and $\{\mathbf{P}^{t}\}_{t=1}^{c_{out}}$. To effectively learning the non-smooth model parameters in each convolution layer, we introduce full-precision filters $\{\mathbf{R}_1,\dots, \mathbf{R}_m\}$ as the proxies of binary filters $\{\mathbf{B}_1,\dots, \mathbf{B}_m\}$ and dense matrices $\{\mathbf{Q}^{t}\}_{t=1}^{c_{out}}$ as the proxies of$\{\mathbf{P}^{t}\}_{t=1}^{c_{out}}$. Instead of directly learning $\{\mathbf{B}_1,\dots, \mathbf{B}_m\}$ and $\{\mathbf{P}^{t}\}_{t=1}^{c_{out}}$, we learn the proxies $\{\mathbf{R}_1,\dots, \mathbf{R}_m\}$ and $\{\mathbf{Q}^{t}\}_{t=1}^{c_{out}}$ during the training. $\{\mathbf{B}_1,\dots, \mathbf{B}_m\}$ and $\{\mathbf{P}^{t}\}_{t=1}^{c_{out}}$ are computed only in the forward propagation and backward propagation. This framework has been successfully used in training binary neural networks \cite{courbariaux2015binaryconnect,hubara2016binarized,rastegari2016xnor}.

\textbf{Forward Propagation.} During the forward propagation, the binary convolution filters $\{\mathbf{B}_1,\dots, \mathbf{B}_m\}$ is obtained by
\begin{equation}\label{eq:B_sign}
\mathbf{B}_i = \text{sign}({\mathbf{R}}_i),
\end{equation}
where sign() is the element-wise sign function which return 1 if the element is larger or equal than zero and return $-1$ otherwise. Similarly, sparse indicator matrices $\{\mathbf{P}^{t}\}_{t=1}^{c_{out}}$ can be obtained by
\begin{equation}\label{eq:P_forwardpass}
\mathbf{P}^t_{ji} =
\begin{cases}
   \ \mathbf{Q}^t_{ji} \ \ \ \ \text{if $j = \text{argmax}\ (|\mathbf{Q}^t_{(:,i)}|)$} \\
   \ 0 \ \ \ \ \ \ \ \ \ \ \text{otherwise}\\
  \end{cases}
\end{equation}
during the forward propagation where the $\text{argmax}\ (|\mathbf{Q}^t_{(:,i)}|)$ function returns the row index $j$ of the maximum absolute value of the $i$-th column of $\mathbf{Q}^t$.

\textbf{Backward Propagation.} Since both the sign() function in (\ref{eq:B_sign}) and the argmax() function in (\ref{eq:P_forwardpass}) are not differentiable, we use the Straight Through Estimator (STE) \citep{bengio2013estimating} to back propagate the estimated gradients for updating the proxy variables $\{\mathbf{R}_1,\dots, \mathbf{R}_m\}$ and $\{\mathbf{Q}^{t}\}_{t=1}^{c_{out}}$. The basic idea of STE is to simply pass the gradients as if the non-differentiable functions sign() and argmax() are not present.

Specifically, let us use $r$ to denote a full-precision weight and it is a proxy for a binary weight $b$. Therefore,
\begin{equation}\label{eq:b_binary}
b =
\begin{cases}
   \ 1 \ \ \ \ \text{if $r \ge 0$,} \\
   \ -1 \ \ \ \ \text{otherwise.}\\
  \end{cases}
\end{equation}
(\ref{eq:b_binary}) is not a differentiable function, STE will just simply estimate its gradient as sign function is not present. That is $\frac{\partial b}{\partial r} = 1$. In practice, we also employ the gradient clipping as in \citet{hubara2016binarized}. Then, the gradient for the sign function is
\begin{equation}\label{eq:gradientClip}
   \frac{\partial b}{\partial r} = 1_{|r|\le 1}.
\end{equation}

Therefore, in the back propagation, the gradient of a convex loss function $L(r)$ with respect to the proxy variable $r$ can be estimated as
\begin{equation}\label{eq:proxy_gradient}
\frac{\partial L(r)}{\partial r} = \frac{\partial L(b)}{\partial b}\frac{\partial b}{\partial r} = \frac{\partial L(b)}{\partial b}1_{|r|\le 1}.
\end{equation}

Similarly, the gradient of a convex loss function $L(\mathbf{Q}^t_{ji})$ with respect to the proxy variable $\mathbf{Q}^t_{ji}$ can be estimated by STE as
\begin{equation}\label{eq:proxy_gradient_Q}
\frac{\partial L(\mathbf{Q}^t_{ji})}{\partial \mathbf{Q}^t_{ji}} = \frac{\partial L(\mathbf{P}^t_{ji})}{\partial \mathbf{P}^t_{ji}}\frac{\partial \mathbf{P}^t_{ji}}{\partial \mathbf{Q}^t_{ji}} = \frac{\partial L(\mathbf{P}^t_{ji})}{\partial \mathbf{P}^t_{ji}}.
\end{equation}

\textbf{Parameter Update.} As shown in (\ref{eq:proxy_gradient}) and (\ref{eq:proxy_gradient_Q}), we now can backpropagate gradients $\frac{\partial L(b)}{\partial b}$ and $\frac{\partial L(\mathbf{P}^t_{ji})}{\partial \mathbf{P}^t_{ji}}$ to their proxies $\{\mathbf{R}_1,\dots, \mathbf{R}_m\}$ and $\{\mathbf{Q}^{t}\}_{t=1}^{c_{out}}$. Then, these two proxy variables can be updated by using any popular optimizer (e.g., SGD with momentum or ADAM \citep{kingma2014adam}). Note that once our training process is completed, we do not need to keep the proxy variables $\{\mathbf{R}_1,\dots, \mathbf{R}_m\}$ and $\{\mathbf{Q}^{t}\}_{t=1}^{c_{out}}$. Only the low-dimensional binary convolution filters$\{\mathbf{B}_1,\dots, \mathbf{B}_m\}$ and the sparse indicator matrices $\{\mathbf{P}^{t}\}_{t=1}^{c_{out}}$ are needed for convolution operations in model inference.

\begin{algorithm}[tb]
\caption{Compressed CNN via stacking low-dimensional binary filters}
\begin{algorithmic}
\STATE \underline{\textbf{Training}}
\STATE \textbf{Input}: training data $\{\mathbf{X}_{train}, \mathbf{y}_{train}\}$, a convex loss function $L(\mathbf{y}, \hat{\mathbf{y}})$, CNN configuration, hyperparameter for low-dimensional binary filter $s$ and $m$
\STATE \textbf{Output}: Compressed CNN model
\end{algorithmic}
\begin{algorithmic}[1]
    \STATE Initialize proxy variables $\{\mathbf{R}_1,\dots, \mathbf{R}_m\}$ and $\{\mathbf{Q}^t\}_{t=1}^{c_{out}}$ for each convolution layer $l$ based on CNN configuration and $s$ and $m$
	\FOR{iter = 1 to maxIter}
	    \STATE Get a minibatch of training data $\{\mathbf{X}, \mathbf{y}\}$
	    \FOR{$l = 1$ to $L$}
	    \STATE Obtain low-dimensional binary filters $\{\mathbf{B}_1,\dots, \mathbf{B}_m\}$ according to (\ref{eq:B_sign})
	    \STATE Obtain $\{\mathbf{P}^t\}_{t=1}^{c_{out}}$ for each convolution filter $t$ according to (\ref{eq:P_forwardpass})
	    \ENDFOR
	\STATE Perform standard \textbf{forward propagation} except that convolution operations are defined in Proposition \ref{proposition:equivalance}% (\ref{eq:reform_conv})
	\STATE Compute the loss $L(\mathbf{y}, \hat{\mathbf{y}})$
	\STATE Perform standard \textbf{backward propagation} except that gradients for $\{\mathbf{R}_1,\dots, \mathbf{R}_m\}$ and  $\{\mathbf{Q}^t\}_{t=1}^{c_{out}}$ are computed respectively as in (\ref{eq:proxy_gradient}) and (\ref{eq:proxy_gradient_Q})
	\STATE Perform \textbf{parameter update} for proxy variables $\{\mathbf{R}_1,\dots, \mathbf{R}_m\}$ and  $\{\mathbf{Q}^t\}_{t=1}^{c_{out}}$ using any popular optimizer (e.g., SGD with momentum or ADAM)
    \ENDFOR
\end{algorithmic}\label{alg:SSBF}
\begin{algorithmic}
    \STATE \underline{\bf{Prediction}}
    \STATE \textbf{Input}: test data $\mathbf{X}_{test}$, Trained compressed CNN
	\STATE \textbf{Output}: predicted labels $\hat{\mathbf{y}}_{test}$;
	\end{algorithmic}
	\begin{algorithmic}[1]
	\STATE Perform standard \textbf{forward propagation} except that convolution operations are defined in Proposition \ref{proposition:equivalance} %(\ref{eq:reform_conv})
	\end{algorithmic}

\end{algorithm}

\subsection{Algorithm Implementation and Analysis}\label{sec:algorithmAnalysis}
We summarize our algorithm in \textbf{Algorithm 1}. In step 1, we initialize the proxy variables $\{\mathbf{R}_1,\dots, \mathbf{R}_m\}$ and $\{\mathbf{Q}^{t}\}_{t=1}^{c_{out}}$ for each convolution layer $l$. From step 4 to step 7, we obtain binary filters $\{\mathbf{B}_1,\dots, \mathbf{B}_m\}$ by (\ref{eq:B_sign}) and sparse indicator matrices $\{\mathbf{P}^t\}_{t=1}^{c_{out}}$ by (\ref{eq:P_forwardpass}) for each convolution layer. In step 8, we perform standard \textbf{forward propagation} except that convolution operations are defined as stacking low-dimensional binary filters. In step 9, we compute the loss $L\{\mathbf{y}, \hat{\mathbf{y}}\}$ using current predicted value $\hat{\mathbf{y}}$ and ground truth $\mathbf{y}$. In step 10, we perform standard \textbf{backward propagation} except that the gradients with respect to proxy variables  $\{\mathbf{R}_1,\dots, \mathbf{R}_m\}$ and  $\{{\mathbf{Q}^t}\}_{t=1}^{c_{out}}$ are computed respectively as in (\ref{eq:proxy_gradient}) and (\ref{eq:proxy_gradient_Q}). In step 11, we perform \textbf{parameter update} for proxy variables using any popular optimizer (e.g., SGD with momentum or ADAM). We implement our \textbf{Algorithm 1} using PyTorch framework \citep{paszke2019pytorch}.

In model inference, we do not need to keep the proxy variables. In each convolution layer, we only need the trained low-dimensional binary filters $\{\mathbf{B}_1,\dots, \mathbf{B}_m\}$ and sparse indicator matrices $\{\mathbf{P}^t\}_{t=1}^{c_{out}}$ to perform convolution operations. Therefore, compared with standard convolution operations using $\mathbf{W}$ as in (\ref{eq:conv}), our proposed method that constructs convolution filter by stacking a number of low-dimensional binary filters can significantly reduce the memory and computation cost of standard CNNs.

With respect to memory cost, for a standard convolution layer, the memory cost is $d \times d \times c_{in} \times c_{out} \times 32$ bits. In our proposed method, the memory cost of storing a set of low-dimensional binary filters $\{\mathbf{B}_1,\dots, \mathbf{B}_m\}$ is $d \times d \times s \times m$ bits. The memory cost of storing stacking parameter $\{\mathbf{P}^t\}_{t=1}^{c_{out}}$ is $\frac{c_{in}}{s}\times m \times c_{out}$ if $\mathbf{P}$ is defined as in (\ref{eq:matrixP}) where each entry can be stored using a single bit and is $\frac{c_{in}}{s}\times c_{out}\times 32 \times 3$ if $\mathbf{P}$ is defined in (\ref{eq:matrixP_2}) \footnote{We use three full-precision vectors to store the indices and values of the nonzero elements in sparse matrix $\mathbf{P}$.}. In our hyperparameter setting, we will set $s = c_{in}f_1$ and $m = c_{out}f_2$ where $f_1$ and $f_2$ are fractional numbers less than 1. In our experiments, we set them as $\frac{1}{2}, \frac{1}{4}, \frac{1}{8}, \dots$, and so on. 
The compression ratio of our proposed method is
\begin{equation}
    \frac{d \times d \times c_{in} \times c_{out} \times 32}{d \times d \times c_{in}f_1 \times c_{out}f_2+\frac{1}{f_1} \times c_{out} \times 32 \times 3} 
\end{equation}
By considering that the memory cost of storing $\{\mathbf{P}^t\}_{t=1}^{c_{out}}$ is relatively small compared with the memory cost of storing low-dimensional binary filters if $f_1$ is not a very small fractional number, the compression ratio of our proposed method can be approximated by $\sim \frac{32}{f_{1}f_{2}}$. The actual compression ratio of our method will be reported in the experimental section.  
% \begin{align}
% &\frac{\times d \times d \times c_{in} \times c_{out} \times 32}{d \times d \times c_{in}f_1 \times c_{out}f_2+\frac{1}{f_1} \times c_{out} \times 32 \times 3} \\
%   & \sim \frac{\times d \times d \times c_{in} \times c_{out} \times 32}{d \times d \times c_{in}f_1 \times c_{out}f_2} = \frac{32}{f_{1}f_{2}}.
% \end{align}
%Note that the memory cost of parameter $P$ can be ignored only if $f_1$ is large enough, thus we also add this term when computing the exact compression ratio in our experiments.
% \begin{equation}\label{eq:compression_ratio_sbf}
%   \frac{\times d \times d \times c_{in} \times c_{out} \times 32}{d \times d \times c_{in}f_1 \times c_{out}f_2+\frac{1}{f_1} \times c_{out} \times 32 \times 3} \\
%   \sim \frac{\times d \times d \times c_{in} \times c_{out} \times 32}{d \times d \times c_{in}f_1 \times c_{out}f_2} = \frac{32}{f_{1}f_{2}}.
% \end{equation}

With respect to computation cost, for a given convolution layer, standard convolution operations require $d \times d \times c_{in} \times w_{out} \times h_{out} \times c_{out}$ FLOPs. In comparison, our method will first require $d \times d \times c_{in} \times w_{out} \times h_{out} \times m$ FLOPs to compute $
\frac{1}{f_1} \times m$ intermediate feature maps where the depth of each intermediate feature map is equal to 1. Then, we select and combine these intermediate feature maps to form the output tensor using $w_{out} \times h_{out} \times \frac{1}{f_1} \times c_{out}$ FLOPs. By considering that $w_{out} \times h_{out} \times \frac{1}{f_1} \times c_{out}$ is relatively small than  $d \times d \times w_{out} \times h_{out} \times c_{in} \times m$ if $f_1$ is not a very small fractional number, the speedup of our model inference can be approximated as
\begin{equation} \label{eq:time_speedup_sbf}
\sim \frac{d \times d \times \times c_{in} \times w_{out} \times h_{out} \times c_{out}}{d \times d \times c_{in} \times w_{out} \times h_{out} \times m} = \frac{1}{f_{2}}.
\end{equation}
Furthermore, due to the binary filters used in our method, convolution operations can be computed using only addition and subtraction (without multiplication) which can further speed up the model inference \cite{rastegari2016xnor}.

\begin{table*}[htbp]
\centering
\caption{Results of different networks on CIFAR-10 and CIFAR-100 datasets using VGG-16 Net}\centering
\begin{tabular}{|c|c|c|c|c}
\hline
Network & Compression Ratio & CIFAR-10 Acc(\%) & CIFAR-100 Acc(\%)   \\
\hline
Full Net (VGG-16) & 1  & \textbf{93.25} & \textbf{73.55}  \\
\hline
LegoNet($f_1 =\frac{1}{4}$ , $f_2= \frac{1}{4}$) & 5.4x  & 91.35 & 70.10 \\
BC   & 31.6x  & 92.11 & 70.64  \\
BWN  & 31.6x & 93.09 & 69.03  \\
BNN  & 31.6x  & 91.21 & 67.88 \\
XNOR-Net  & 31.6x  & 90.02 & 68.63  \\
\hline
%{Ours ($f_1$ = 1, $f_2$ = 1)} & $\sim32$ & 91.66   \\
{SLBF ($f_1$ = 1, $f_2 = \frac{1}{2}$)} & 60.1x & 91.44 & 68.80 \\
{SLBF ($f_1 = \frac{1}{2}$, $f_2 = \frac{1}{2}$)} & 103.2x & 91.30 & 67.55 \\
%{SLBF ($f_1$ = 1, $f_2 = \frac{1}{4}$)} & 113.56 & 89.68 & 64.17 &0.500 \\
%{Ours ($f_1 = \frac{1}{2}$, $f_2$ = 0.5)(ADAM)} & $\sim128$ & 91.22   \\
{SLBF ($f_1 = \frac{1}{2}$, $f_2 = \frac{1}{4}$)} & 173.1x & 90.24 & 66.68  \\
{SLBF ($f_1 = \frac{1}{2}$, $f_2 = \frac{1}{8}$)} &  261.4x & 89.24 & 62.88 \\
%{SLBF ($f_1 = \frac{1}{4}$, $f_2 = \frac{1}{4}$)} & 179.3x & 89.92 & 65.97 & 0.317\\
%{SLBF ($f_1 = \frac{1}{4}$, $f_2 = \frac{1}{8}$)} & 217.3x & 88.63 & 62.65 & 0.262   \\
\hline
\end{tabular}\label{tt}
\end{table*}

\section{Experiments}\label{sec:experiments}
In this section, we compare the performance of our proposed method with five state-of-the-art CNN model compression algorithms on two benchmark image classification datasets: CIFAR-10 and CIFAR-100 \citep{krizhevsky2009learning}. %Similar to \citet{yang2019legonet}, 
Note that we focus on the compressing convolution layers as in \citet{yang2019legonet}. The full connection layers can be compressed by adaptive fastfood transform \cite{yang2015deep} which is beyond the scope of this paper. %when computing the memory cost and compression ratio. 
We also evaluate the performance of these algorithms on MNIST \citep{lecun1998mnist}  dataset. The results on MNIST dataset can be found in supplementary materials due to page limitation. %The following subsections demonstrate the results on CIFAR-10 and CIFAR-100 datasets, and analysis on MNIST can be found in 

In our experiments, we evaluate the performance of the following seven algorithms:
\begin{itemize}[leftmargin=*]
\item Full Net: deep CNN model with full-precision weights;
\item BinaryConnect(BC): deep CNN model with binary weights  \citep{courbariaux2015binaryconnect};
\item Binarized Neural Networks(BNN): deep CNN model with both binary weights and binary activations \citep{hubara2016binarized};
\item Binary Weight Network(BWN): similar to BC but scaling factors are added to binary filters \citep{rastegari2016xnor};
\item XNOR-Networks(XNOR-Net): similar to BNN but scaling factors are added to binary filters and binary activations \citep{rastegari2016xnor};
\item LegoNet: Efficient CNN with Lego filters \citep{yang2019legonet}
\item Stacking Low-dimensional Binary Filters (SLBF): Our proposed method.
\end{itemize}

\subsection{Experimental Results on CIFAR-10 and CIFAR-100 using VGG-16 Net}

We first present our experiment settings and results on CIFAR-10 and CIFAR-100 datasets by using VGG-16 \citep{simonyan2014very} network as the CNN architecture. CIFAR-10 consists of 50,000 training samples and 10,000 test samples with 10 classes while CIFAR-100 contains more images belonging to 100 classes. Each sample in these two datasets is a $32 \times 32$ colour image. %We use VGG-16 \citep{simonyan2014very} network as the CNN architecture for this dataset in this section. 
The VGG-16 network contains 13 convolution layers and 3 full-connected layers. We use this CNN network architecture for all seven methods. The batch normalization with scaling and shifting applies to all methods too. In our method SLBF, SGD with the momentum of 0.9 is used as the optimizer. %We also evaluate our method using ADAM as the optimizer which produces similar results and we will discuss it supple. 
For other five model compression methods, we use the suggested settings from their papers. %In our method, there are two important hyperparameters $f_1$ and $f_2$ which control the compression ratio of our compressed CNN model.

Our experimental results with different settings of $f_1$ and $f_2$ using VGG-16 are presented in Table \ref{tt}. Note that we only report the result for LegoNet with $f_1=\frac{1}{4}$ and $f_2 =\frac{1}{4}$ because it gets the best trade-off between compression ratio and accuracy based on our experimental results. The VGG-16 with full precision weights gets the highest accuracy $93.25\%$ on CIFAR-10 and $73.55\%$ on CIFAR-100. For CIFAR-10, our method can get $91.30\%$ with model compression ratio 103.2x. This is encouraging since we can compress the full model by more than 100 times without sacrifice classification accuracy too much ($<2\%$). The loss of accuracy with the same compression ratio is larger on CIFAR-100 but the performance is still comparable with other benchmark methods. As expected, the accuracy of our method will decrease when compression ratio increases. However, as can be seen from Table \ref{tt}, the accuracy does not decrease much (i.e., from 91.30\% to 88.63\%) even we increase the compression ratio from 103.2x to 217.32x. It clearly demonstrates our proposed method can achieve a good trade-off between accuracy and model compression ratio. %As for the memory cost, the cost of saving the full-precision scaling factors we introduced will be larger as $f_2$ increases and the change of $f_1$ will not effect the memory cost greatly. This condition is more explicit when there are more input and output channels.

\begin{table*}[bth]
\centering
\caption{Results of different networks on CIFAR-10 and CIFAR-100 datasets using ResNet-18 Net}\centering
\begin{tabular}{|c|c|c|c|}
\hline
Network & Compression Ratio & CIFAR-10 Acc(\%) & CIFAR-100 Acc(\%) \\
\hline
Full Net (ResNet-18) & 1  & \textbf{95.19} & \textbf{77.11} \\
\hline
LegoNet($f_1 =\frac{1}{4}$ , $f_2= \frac{1}{4}$) & 17.5x  & 93.55 & 72.67 \\
BC                & 31.8x  & 93.73 & 71.15  \\
BWN              & 31.8x   & 93.97 & 72.92   \\
BNN               & 31.8x   & 90.47 & 70.34  \\
XNOR-Net          & 31.8x   & 90.14 & 72.87  \\
\hline
%{Ours ($f_1$ = 1, $f_2$ = 1)} & $\sim32$ & 91.66   \\
{SLBF ($f_1$ = 1, $f_2 = \frac{1}{2}$)} & 58.7x & 93.82 &  74.59\\
{SLBF ($f_1 = \frac{1}{2}$, $f_2 = \frac{1}{2}$)} & 95.1x & 93.72 & 74.19  \\
{SLBF ($f_1$ = 1, $f_2 = \frac{1}{4}$)} & 108.2x & 92.96 &  72.12\\
{SLBF ($f_1 = \frac{1}{2}$, $f_2 = \frac{1}{4}$)} &  151.4x & 92.94 & 71.91  \\
{SLBF ($f_1 = \frac{1}{2}$, $f_2 = \frac{1}{8}$)} &  214.9x & 91.70 & 67.89 \\
%{SLBF ($f_1 = \frac{1}{4}$, $f_2 = \frac{1}{4}$)} & 138.9x & 92.05 & 68.63 & 0.307\\
%{SLBF ($f_1 = \frac{1}{4}$, $f_2 = \frac{1}{8})$)} & 160.7x & 90.94 & 67.59 & 0.265 \\
\hline
\end{tabular}\label{100r}
\end{table*}

\subsection{Experimental Results on CIFAR-10 and CIFAR-100 using ResNet-18}

We also apply the recent ResNet-18 \citep{he2016deep} structure with 17 convolution layers followed by one full-connection layer on CIFAR-10 and CIFAR-100 datasets. Similar to the experimental setting using VGG-16, SGD with the momentum of 0.9 is used as the optimizer in our method. 

The accuracy and compression ratio of benchmark and our method with different settings using ResNet-18 is shown in Table \ref{100r}. Generally the ResNet performs better than VGG-16 network on these two datasets, it can obtain a comparable accuracy of 74.19\% on CIFAR-100 with about 95 times compression when setting $f_1 = \frac{1}{2}$ and $f_2 = \frac{1}{2}$, and the accuracy will not decrease greatly as compression ratio increases to 151 times. In the following subsections, we empirically explore the impact of scaling factors and several other training techniques on our proposed method.

% \subsection{Experimental Results on CIFAR-100}
% Compared to CIFAR-10, CIFAR-100 contains more images with 100 classes, thus we evaluate our method on this dataset using two different popular networks. The first VGG-16 network has the same structure as described in the experiments on CIFAR-10. Then we also apply ResNet-18 \citep{he2016deep} with 17 $3\times3$ convolution layers followed by one full-connection layer. The accuracy and compression ratio of benchmark and our method with different setting using VGG-16 and ResNet-18 is shown in Table \ref{tt} and Table \ref{100r} respectively. Generally the ResNet performs better than VGG network on this dataset, it can obtain a comparable accuracy of 74.19\% with about 95 times compression when setting $f_1 = \frac{1}{2}$ and $f_2 = \frac{1}{2}$, and the accuracy will not decrease greatly as compression ratio increases to 151 times. In the following subsections, we empirically explore the impact of scaling factors and several other training techniques on our proposed method.

\begin{figure}[bht]\centering
%\subfigure[MNIST]{\centering \label{mnist}\includegraphics[width=0.49\columnwidth]{sf-mnist.png}}
%\subfigure[CIFAR-10]{
\centering \includegraphics[width=0.85\columnwidth]{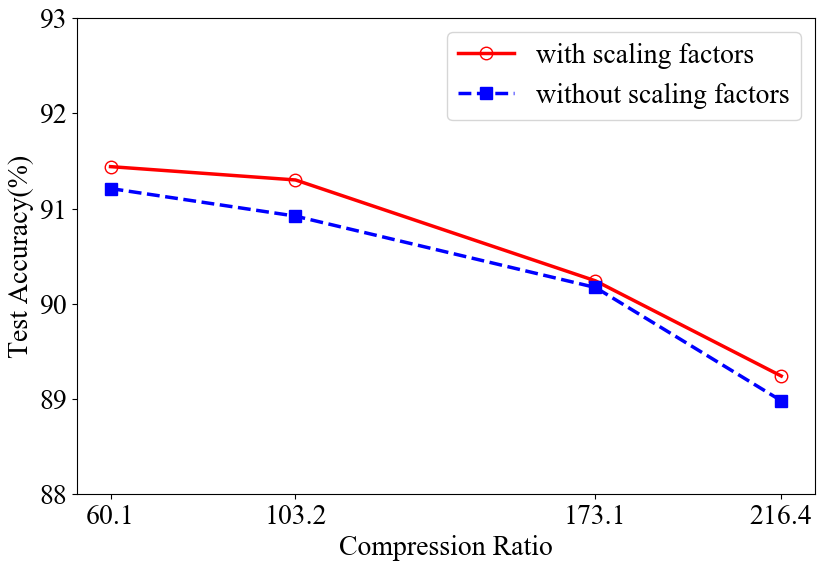}
%}
\caption{Comparison of our method with and without scaling factors}
\label{cifar10}
\end{figure}

\subsection{The Impact of Scaling Factors in Matrix $\mathbf{P}$} %$\{\mathbf{P}^t\}_{t=1}^{t=c_{out}}$}
In our proposed method, the matrix $\mathbf{P}$ used for selecting and stacking binary filters can be defined either as in (\ref{eq:matrixP}) or as in (\ref{eq:matrixP_2}). The difference between these two definitions is that (\ref{eq:matrixP_2}) will multiply binary filters with scaling factors when stacking them together. In here, we evaluate the impact of scaling factors in our method. We compare the accuracy of our method with and without scaling factors on CIFAR-10 datasets using VGG-16 as the compression ratio changing from 60.1x to 217.3x and the results are shown in Figure \ref{cifar10}. As can be seen from Figure \ref{cifar10}, our proposed method with scaling factors always gets slightly higher accuracy than without scaling factors.
% \begin{figure}[bth]\centering
% \subfigure[MNIST]{\centering \label{mnist}\includegraphics[width=0.49\columnwidth]{sf-mnist.png}}
% \subfigure[CIFAR-10]{\centering \label{Fig.2.2}\includegraphics[width=0.5\columnwidth]{sf-cifar.png}}
% \caption{Comparison of our method with and without scaling factors}
% \label{cifar10}
% \end{figure}

\subsection{The Impact of Batch Normalization}
Batch normalization \citep{ioffe2015batch} is a popular technique to improve the training of deep neural networks. It standardizes the inputs to a layer for each mini-batch. We compare the performance of our proposed method with two different batch normalization settings: (1) batch normalization without scaling and shifting: normalize inputs to have zero mean and unit variance; (2) batch normalization with scaling and shifting. The results are reported in Figure \ref{batchnorm} and it shows that batch normalization with scaling obtains better accuracy than without scaling and shifting on CIFAR-10 dataset. Thus we apply these two factors on our methods in the experiments.

We also investigate the impact of other commonly used techniques for deep learning training in our model, such as different optimizer and different batch size. More results and detailed discussion can be found in supplementary materials. 
% \begin{figure}[hbt]\centering
% \subfigure[MNIST]{ \label{bnmnist}
% \includegraphics[width=0.45\columnwidth]{bnmnist.png}}
% \subfigure[CIFAR-10]{\label{bncifar}
% \includegraphics[width=0.48\columnwidth]{batchnorm.png}}
% \caption{Accuracy of batch normalization with/without scaling and shifting ($f_1=\frac{1}{2}, f_2=\frac{1}{2}$).}
% \label{batchnorm}
% \end{figure}
\begin{figure}[bt]\centering
%\subfigure[MNIST]{ \label{bnmnist}
%\includegraphics[width=0.45\columnwidth]{bnmnist.png}}
%\subfigure[CIFAR-10]{\label{bncifar}
\includegraphics[width=0.95\columnwidth]{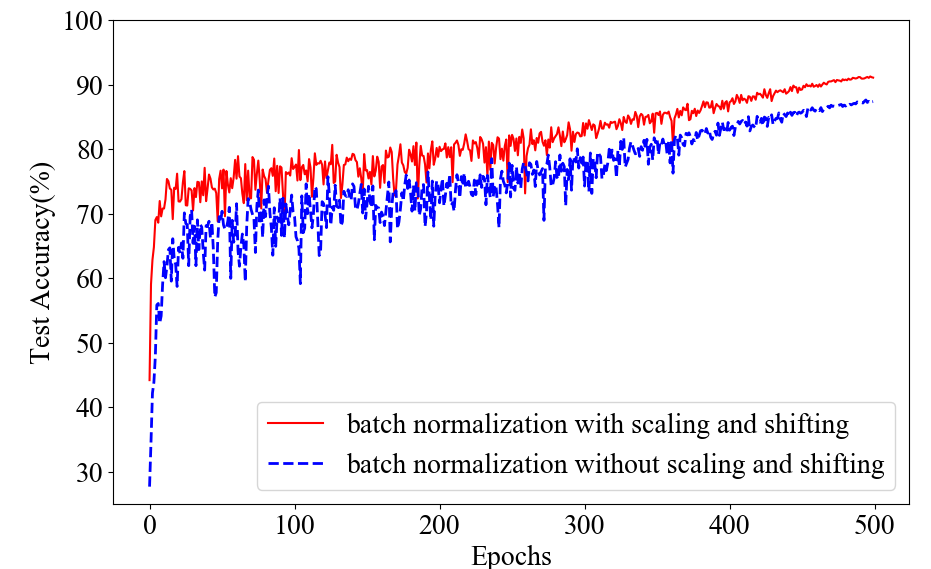}
%}
\caption{Accuracy of batch normalization with/without scaling and shifting ($f_1=\frac{1}{2}, f_2=\frac{1}{2}$).}
\label{batchnorm}
\end{figure}

\section{Conclusions and Future Works}
In this paper, we propose a novel method to compress deep CNN by selecting and stacking low-dimensional binary filters. Our proposed method can overcome the theoretical compression ratio limit of existing binary CNN models. We have theoretically shown that our proposed model is equivalent to select and stack low-dimensional feature maps generated by low-dimensional binary filters and therefore can be efficiently trained using the split-transform-merge strategy. We also provide detailed analysis on the memory and computation cost of our model for model inference. We compare our proposed method with other five popular model compression techniques on three benchmark datasets. Our experimental results clearly demonstrate that our proposed method can achieve comparable accuracy with much higher compression ratio. In our experiments, we also empirically explore the impact of various training techniques on our proposed method.

In the future, we will consider to use binary activation function. By doing it, convolution operations in each layer will be replaced by cheap XNOR and POPCOUNT binary operations which can further speed up model inference as observed in \citep{rastegari2016xnor}. We are also interested in investigating alternative methods to Straight Through Estimator (STE) for learning non-smooth model parameters.

% \newpage
\begin{small}
\bibliography{aaai}

\begin{thebibliography}{29}
\providecommand{\natexlab}[1]{#1}
\providecommand{\url}[1]{\texttt{#1}}
\providecommand{\urlprefix}{URL }
\expandafter\ifx\csname urlstyle\endcsname\relax
  \providecommand{\doi}[1]{doi:\discretionary{}{}{}#1}\else
  \providecommand{\doi}{doi:\discretionary{}{}{}\begingroup
  \urlstyle{rm}\Url}\fi

\bibitem[{Alizadeh et~al.(2019)Alizadeh, Fern{\'{a}}ndez{-}Marqu{\'{e}}s, Lane,
  and Gal}]{AlizadehFLG19an}
Alizadeh, M.; Fern{\'{a}}ndez{-}Marqu{\'{e}}s, J.; Lane, N.~D.; and Gal, Y.
  2019.
\newblock An Empirical study of Binary Neural Networks' Optimisation.
\newblock In \emph{7th International Conference on Learning Representations,
  {ICLR} 2019, New Orleans, LA, USA, May 6-9, 2019}.

\bibitem[{Ba and Caruana(2014)}]{ba2014deep}
Ba, J.; and Caruana, R. 2014.
\newblock Do deep nets really need to be deep?
\newblock In \emph{Advances in neural information processing systems},
  2654--2662.

\bibitem[{Bengio, L{\'e}onard, and Courville(2013)}]{bengio2013estimating}
Bengio, Y.; L{\'e}onard, N.; and Courville, A. 2013.
\newblock Estimating or propagating gradients through stochastic neurons for
  conditional computation.
\newblock \emph{arXiv preprint arXiv:1308.3432} .

\bibitem[{Buciluǎ, Caruana, and Niculescu-Mizil(2006)}]{bucilua2006model}
Buciluǎ, C.; Caruana, R.; and Niculescu-Mizil, A. 2006.
\newblock Model compression.
\newblock In \emph{Proceedings of the 12th ACM SIGKDD international conference
  on Knowledge discovery and data mining}, 535--541.

\bibitem[{Cheng et~al.(2018)Cheng, Wang, Li, Hu, and Lu}]{cheng2018recent}
Cheng, J.; Wang, P.-s.; Li, G.; Hu, Q.-h.; and Lu, H.-q. 2018.
\newblock Recent advances in efficient computation of deep convolutional neural
  networks.
\newblock \emph{Frontiers of Information Technology \& Electronic Engineering}
  19(1): 64--77.

\bibitem[{Courbariaux, Bengio, and David(2015)}]{courbariaux2015binaryconnect}
Courbariaux, M.; Bengio, Y.; and David, J.-P. 2015.
\newblock Binaryconnect: Training deep neural networks with binary weights
  during propagations.
\newblock In \emph{Advances in neural information processing systems},
  3123--3131.

\bibitem[{Denton et~al.(2014)Denton, Zaremba, Bruna, LeCun, and
  Fergus}]{denton2014exploiting}
Denton, E.~L.; Zaremba, W.; Bruna, J.; LeCun, Y.; and Fergus, R. 2014.
\newblock Exploiting linear structure within convolutional networks for
  efficient evaluation.
\newblock In \emph{Advances in neural information processing systems},
  1269--1277.

\bibitem[{Gupta et~al.(2015)Gupta, Agrawal, Gopalakrishnan, and
  Narayanan}]{gupta2015deep}
Gupta, S.; Agrawal, A.; Gopalakrishnan, K.; and Narayanan, P. 2015.
\newblock Deep learning with limited numerical precision.
\newblock In \emph{International Conference on Machine Learning}, 1737--1746.

\bibitem[{Han, Mao, and Dally(2015)}]{han2015deep}
Han, S.; Mao, H.; and Dally, W.~J. 2015.
\newblock Deep compression: Compressing deep neural networks with pruning,
  trained quantization and huffman coding.
\newblock \emph{arXiv preprint arXiv:1510.00149} .

\bibitem[{Han et~al.(2015)Han, Pool, Tran, and Dally}]{han2015learning}
Han, S.; Pool, J.; Tran, J.; and Dally, W. 2015.
\newblock Learning both weights and connections for efficient neural network.
\newblock In \emph{Advances in neural information processing systems},
  1135--1143.

\bibitem[{He et~al.(2016)He, Zhang, Ren, and Sun}]{he2016deep}
He, K.; Zhang, X.; Ren, S.; and Sun, J. 2016.
\newblock Deep residual learning for image recognition.
\newblock In \emph{Proceedings of the IEEE conference on computer vision and
  pattern recognition}, 770--778.

\bibitem[{Hinton, Vinyals, and Dean(2015)}]{hinton2015distilling}
Hinton, G.; Vinyals, O.; and Dean, J. 2015.
\newblock Distilling the knowledge in a neural network.
\newblock \emph{arXiv preprint arXiv:1503.02531} .

\bibitem[{Howard et~al.(2017)Howard, Zhu, Chen, Kalenichenko, Wang, Weyand,
  Andreetto, and Adam}]{howard2017mobilenets}
Howard, A.~G.; Zhu, M.; Chen, B.; Kalenichenko, D.; Wang, W.; Weyand, T.;
  Andreetto, M.; and Adam, H. 2017.
\newblock Mobilenets: Efficient convolutional neural networks for mobile vision
  applications.
\newblock \emph{arXiv preprint arXiv:1704.04861} .

\bibitem[{Hubara et~al.(2016)Hubara, Courbariaux, Soudry, El-Yaniv, and
  Bengio}]{hubara2016binarized}
Hubara, I.; Courbariaux, M.; Soudry, D.; El-Yaniv, R.; and Bengio, Y. 2016.
\newblock Binarized neural networks.
\newblock In \emph{Advances in neural information processing systems},
  4107--4115.

\bibitem[{Ioffe and Szegedy(2015)}]{ioffe2015batch}
Ioffe, S.; and Szegedy, C. 2015.
\newblock Batch normalization: Accelerating deep network training by reducing
  internal covariate shift.
\newblock \emph{arXiv preprint arXiv:1502.03167} .

\bibitem[{Jaderberg, Vedaldi, and Zisserman(2014)}]{jaderberg2014speeding}
Jaderberg, M.; Vedaldi, A.; and Zisserman, A. 2014.
\newblock Speeding up Convolutional Neural Networks with Low Rank Expansions.
\newblock In \emph{Proceedings of the British Machine Vision Conference. BMVA
  Press}.

\bibitem[{Kingma and Ba(2014)}]{kingma2014adam}
Kingma, D.~P.; and Ba, J. 2014.
\newblock Adam: A method for stochastic optimization.
\newblock \emph{arXiv preprint arXiv:1412.6980} .

\bibitem[{Krishnamoorthi(2018)}]{krishnamoorthi2018quantizing}
Krishnamoorthi, R. 2018.
\newblock Quantizing deep convolutional networks for efficient inference: A
  whitepaper.
\newblock \emph{arXiv preprint arXiv:1806.08342} .

\bibitem[{Krizhevsky, Hinton et~al.(2009)}]{krizhevsky2009learning}
Krizhevsky, A.; Hinton, G.; et~al. 2009.
\newblock Learning multiple layers of features from tiny images .

\bibitem[{LeCun et~al.(1998)LeCun, Bottou, Bengio, and
  Haffner}]{lecun1998gradient}
LeCun, Y.; Bottou, L.; Bengio, Y.; and Haffner, P. 1998.
\newblock Gradient-based learning applied to document recognition.
\newblock \emph{Proceedings of the IEEE} 86(11): 2278--2324.

\bibitem[{LeCun, Cortes, and Burges(1998)}]{lecun1998mnist}
LeCun, Y.; Cortes, C.; and Burges, C.~J. 1998.
\newblock The MNIST database of handwritten digits, 1998.
\newblock \emph{URL http://yann. lecun. com/exdb/mnist} 10: 34.

\bibitem[{Paszke et~al.(2019)Paszke, Gross, Massa, Lerer, Bradbury, Chanan,
  Killeen, Lin, Gimelshein, Antiga et~al.}]{paszke2019pytorch}
Paszke, A.; Gross, S.; Massa, F.; Lerer, A.; Bradbury, J.; Chanan, G.; Killeen,
  T.; Lin, Z.; Gimelshein, N.; Antiga, L.; et~al. 2019.
\newblock PyTorch: An imperative style, high-performance deep learning library.
\newblock In \emph{Advances in Neural Information Processing Systems},
  8024--8035.

\bibitem[{Rastegari et~al.(2016)Rastegari, Ordonez, Redmon, and
  Farhadi}]{rastegari2016xnor}
Rastegari, M.; Ordonez, V.; Redmon, J.; and Farhadi, A. 2016.
\newblock Xnor-net: Imagenet classification using binary convolutional neural
  networks.
\newblock In \emph{European conference on computer vision}, 525--542. Springer.

\bibitem[{Simonyan and Zisserman(2014)}]{simonyan2014very}
Simonyan, K.; and Zisserman, A. 2014.
\newblock Very deep convolutional networks for large-scale image recognition.
\newblock \emph{arXiv preprint arXiv:1409.1556} .

\bibitem[{Szegedy et~al.(2015)Szegedy, Liu, Jia, Sermanet, Reed, Anguelov,
  Erhan, Vanhoucke, and Rabinovich}]{szegedy2015going}
Szegedy, C.; Liu, W.; Jia, Y.; Sermanet, P.; Reed, S.; Anguelov, D.; Erhan, D.;
  Vanhoucke, V.; and Rabinovich, A. 2015.
\newblock Going deeper with convolutions.
\newblock In \emph{Proceedings of the IEEE conference on computer vision and
  pattern recognition}, 1--9.

\bibitem[{Ullrich, Meeds, and Welling(2017)}]{ullrich2017soft}
Ullrich, K.; Meeds, E.; and Welling, M. 2017.
\newblock Soft weight-sharing for neural network compression.
\newblock \emph{arXiv preprint arXiv:1702.04008} .

\bibitem[{Yang et~al.(2015)Yang, Moczulski, Denil, de~Freitas, Smola, Song, and
  Wang}]{yang2015deep}
Yang, Z.; Moczulski, M.; Denil, M.; de~Freitas, N.; Smola, A.; Song, L.; and
  Wang, Z. 2015.
\newblock Deep fried convnets.
\newblock In \emph{Proceedings of the IEEE International Conference on Computer
  Vision}, 1476--1483.

\bibitem[{Yang et~al.(2019)Yang, Wang, Liu, Chen, Xu, Shi, Xu, and
  Xu}]{yang2019legonet}
Yang, Z.; Wang, Y.; Liu, C.; Chen, H.; Xu, C.; Shi, B.; Xu, C.; and Xu, C.
  2019.
\newblock Legonet: Efficient convolutional neural networks with lego filters.
\newblock In \emph{International Conference on Machine Learning}, 7005--7014.

\bibitem[{Zhang et~al.(2018)Zhang, Zhou, Lin, and Sun}]{zhang2018shufflenet}
Zhang, X.; Zhou, X.; Lin, M.; and Sun, J. 2018.
\newblock Shufflenet: An extremely efficient convolutional neural network for
  mobile devices.
\newblock In \emph{Proceedings of the IEEE Conference on Computer Vision and
  Pattern Recognition}, 6848--6856.

\end{thebibliography}
\end{small}
%\bibliography{myRef}

\newpage
\section{Appendix}
\section{Proof of Proposition 1}
%\section{Appendix}

\begin{proposition}
Suppose $\mathbf{X}_{input(i)} \in \mathbb{R}^{w_{in} \times h_{in} \times s}$, $\{\mathbf{B}_1,\dots, \mathbf{B}_m\}$ is a set of $m$ low-dimensional binary filters where each $\mathbf{B}_i \in \mathbb{R}^{d \times d \times s}$ and $\mathbf{P}_{(:,i)}^{t}$ is the $i$-th column in $\mathbf{P}^t$ which is a length-$m$ sparse vector with only one non-zero element. Then, $\emph{Conv}(\mathbf{X}_{input(i)}, \sum_{j=1}^{m}\mathbf{P}_{ji}^{t}\mathbf{B}_j)$ is equivalent to $\sum_{j=1}^{m}\mathbf{P}_{ji}^{t}\emph{Conv}(\mathbf{X}_{input(i)}, \mathbf{B}_j)$.
\notag
\end{proposition}
\begin{proof}
Let us divide $\mathbf{X}_{input(i)}$ into $p = w_{out} \times h_{out}$ patches and each patch $\mathbf{G}_j$ is with size $\mathbb{R}^{d \times d \times s}$. We can vectorize each patch and form a matrix $\widetilde{\mathbf{X}} = [\text{vec}(\mathbf{G}_1), \text{vec}(\mathbf{G}_2), \dots, \text{vec}(\mathbf{G}_p)] \in \mathbb{R}^{d^2s \times p}$. Similarly, we can vectorize each low-dimensional binary convolution filter $\mathbf{B}_i$ and form a matrix $\widetilde{\mathbf{B}} = [\text{vec}(\mathbf{B}_1), \text{vec}(\mathbf{B}_2), $ $\dots, \text{vec}(\mathbf{B}_m)] \in \mathbb{R}^{d^2s \times m}$. Based on the definition of convolution operation and applying the associative property of matrix multiplication, we have
\begin{equation}\label{eq:reform_conv}
\text{Conv}(\mathbf{X}_{input(i)}, \sum_{j=1}^{m}\mathbf{P}_{ji}^{t}\mathbf{B}_j) =  \widetilde{\mathbf{X}}^{T}(\widetilde{\mathbf{B}}\mathbf{P}_{(:,i)}^{t}) = (\widetilde{\mathbf{X}}^{T}\widetilde{\mathbf{B}})\mathbf{P}_{(:,i)}^{t}.
\end{equation}
Note that $(\widetilde{\mathbf{X}}^{T}\widetilde{\mathbf{B}})\mathbf{P}_{(:,i)}^{t}$ in (\ref{eq:reform_conv}) can be rewritten as $\sum_{j=1}^{m}\mathbf{P}_{ji}^{t}\text{Conv}(\mathbf{X}_{input(i)}, \mathbf{B}_j)$ which can be interpreted as we first perform convolution operations on $\mathbf{X}_{input(i)}$ using  $\{\mathbf{B}_1,\dots, \mathbf{B}_m\}$ to generate $m$ intermediate feature maps $\{\text{Conv}(\mathbf{X}_{input(i)},$ $\mathbf{B}_j)\}_{j=1}^{m}$ and then select one feature map from them using sparse vector $\mathbf{P}^t_{(:,i)}$.
\end{proof}

\section{Additional Experimental Results}

\subsection{Results on MNIST using LeNet-5}

\begin{table}[htp]
% \normalsize
\small
\centering
\caption{Results of different networks on MNIST dataset}\centering
\begin{tabular}{|c|c|c|}
\hline
Network & Compression Ratio & Acc(\%) \\
\hline
Full Net (LeNet-5)  & 1 & 99.48 \\
\hline
LegoNet ($f_1 = \frac{1}{4}$, $f_2 = \frac{1}{4}$)  & 15.7  & 99.34  \\
BC       & $\sim32$  & 98.82 \\
BWN      & $\sim32$  & \textbf{99.38}\\
BNN      & $\sim32$  & 98.60 \\
XNOR-Net & $\sim32$ & 99.21  \\
\hline
%{Ours ($k$ = 1, $f_2$ = 1)} & $\sim32$ & 99.34  \\
% {SLBF ($f_1 = 1$, $f_2= \frac{1}{2}$)} & 18.61 & 99.32& 2.781 \\
{SLBF ($f_1 = 1$, $f_2 = \frac{1}{4}$)} & 21.65  & 99.27 \\
{SLBF ($f_1 = 1$, $f_2 = \frac{1}{8}$)} &23.57 & 99.09  \\
{SLBF ($f_1 = 1$, $f_2 = \frac{1}{16}$)} &24.67 & 98.96\\
% {SLBF ($f_1 = \frac{1}{2}$, $f_2 = \frac{1}{2}$)} & 18.71 & 99.30 \\
% {SLBF ($f_1 = \frac{1}{2}$, $f_2 = \frac{1}{4}$)} & 20.13 & 99.23 \\
% {SLBF ($f_1 = \frac{1}{4}$, $f_2 = \frac{1}{4}$)} & 16.06 & 99.22 \\
% {SLBF ($f_1 = \frac{1}{4}$, $f_2 = \frac{1}{8}$)} & 16.31  & 99.20\\
\hline
\end{tabular}\label{t1}
\end{table}

We present our experiment settings and results on MNIST dataset \citep{lecun1998mnist} in this section. MNIST dataset consists of 60,000 training samples and 10,000 test samples. Each sample is a \(28 \times 28\) pixel grayscale handwritten digital image. The convolutional network architecture we used for MNIST data is the LeNet-5 \citep{lecun1998gradient} which has two convolution layers followed by a MaxPooling layer and two full-connection layers. We use the same CNN architecture for all seven methods. The setting of batch normalization and optimizer is the same as experiments on VGG-16 and ResNet-18.

According to the results in Table \ref{t1}, all methods can obtain $>98\%$ accuracy on this dataset and the difference among them is very small in regard to classification accuracy. The compression ratio for LegoNet is not high since it uses full-precision low-dimensional filters for stacking, thus our methods obtain higher compression ratio compared to LegoNet. However, the compression ratio is still less than binary networks because of the full-precision scaling factors we used on this dataset. As shown in the experiments on CIFAR-10 and CIFAR-100, the compression ratio has been improved greatly when applying deeper network.

\begin{figure}[hb]\centering
\includegraphics[scale=0.33]{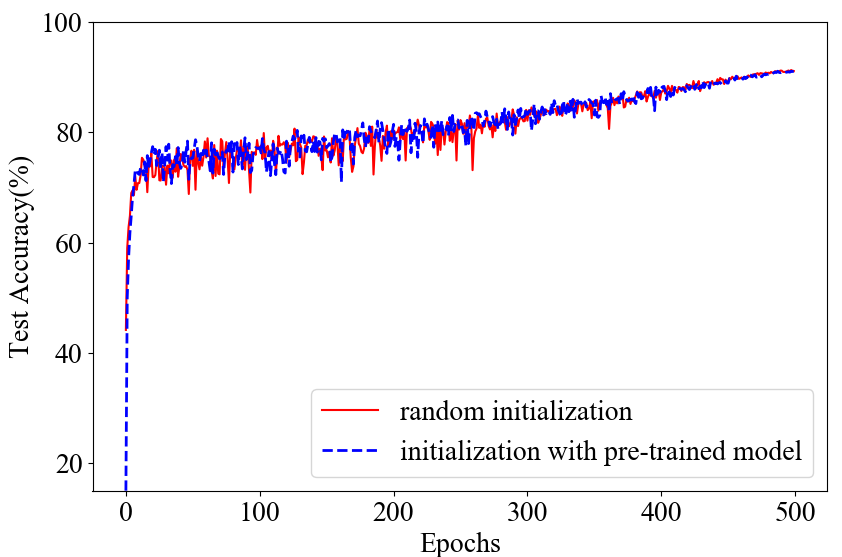}
\caption{Comparison of our method ($f_1=\frac{1}{2}, f_2=\frac{1}{2}$) with different parameter initialization techniques}
\label{initialization}
\end{figure}

\subsection{The Impact of Initialization}
In the step 1 of our training algorithm, we can either randomly initialize the proxy variables $\{\mathbf{R}_1,\dots, \mathbf{R}_m\}$ or initialize them from a pre-trained full-precision LegoNet model. In here, we compare the accuracy of our method with random initialization and initialization with pre-trained model on CIFAR-10 dataset applying VGG-16. As shown in Figure \ref{initialization}, even though initialization from pre-trained model can achieve significant higher accuracy than random initialization in the very beginning, both two initialization methods yield very similar accuracy after a certain number of epochs.

\begin{figure}[ht]\centering
%\subfigure[Accuracy]{\centering \label{acc}
\includegraphics[scale=0.33]{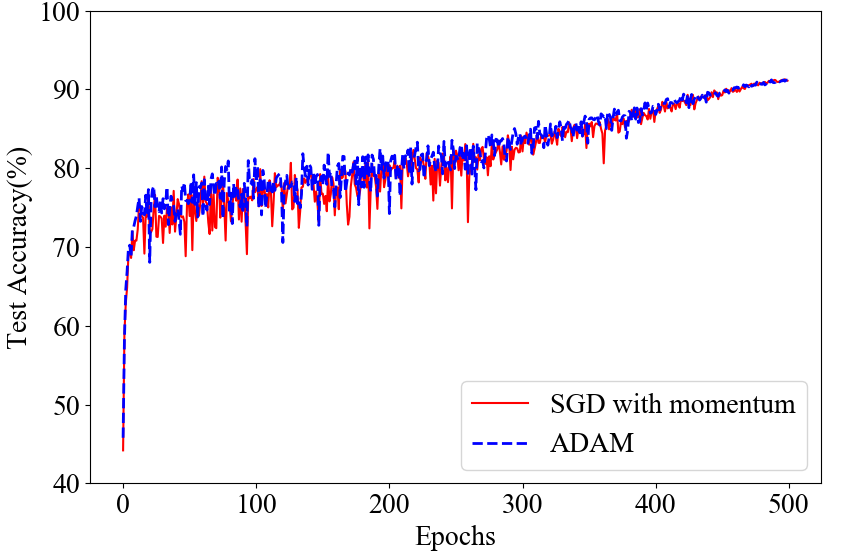}%}
%\subfigure[Test Loss]{\centering \label{testloss}\includegraphics[width=0.5\columnwidth]{optestloss.png}}
\caption{Comparison of our method ($f_1=\frac{1}{2}, f_2=\frac{1}{2}$) using different optimizers}
\label{optimizer}
\end{figure}

\begin{figure}[ht]\centering
\includegraphics[scale=0.33]{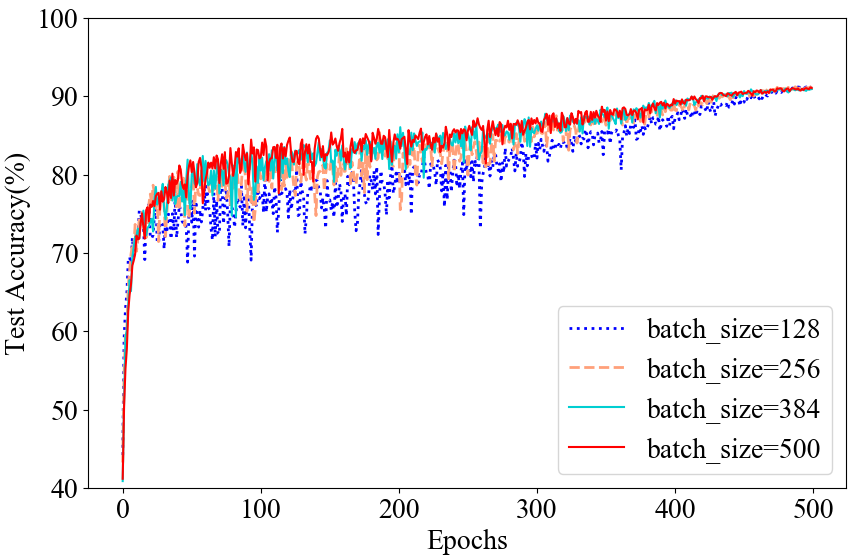}
\caption{Comparison of our method ($f_1=\frac{1}{2}, f_2=\frac{1}{2}$) using different batch\_size}
\label{batchsize}
\end{figure}

\subsection{The Impact of Optimizer}
Our default optimizer is SGD with the momentum of 0.9. We also evaluate the performance of our method using another popular optimizer ADAM\citep{kingma2014adam}. Figure \ref{optimizer} compares the test accuracy of our method using SGD with momentum and ADAM on CIFAR-10 dataset with $f_1 = \frac{1}{2}$ and $f_2 = \frac{1}{2}$. As can be observed from Figure \ref{optimizer}, these two optimizers obtain similar results in the end.

\subsection{The Impact of Batch\_Size}
We also evaluate the impact of batch size in our method. We compare the accuracy with four different settings of batch size. As shown in Figure \ref{batchsize}, we can observe that larger batch size gets better accuracy in the beginning but they reach a similar accuracy in the end.

% % \begin{small}
% % %\bibliographystyle{aaai}
% % \bibliography{aaai}
% % \end{small}

\end{document}